\documentclass[letterpaper, 10 pt,journal, twoside]{./support/IEEEtran}
\usepackage{times}
\usepackage[pdftex]{graphicx}
\usepackage{subfigure}
\usepackage{amsmath, amsthm}
\usepackage{amssymb,amsopn,amstext,amsfonts}
\usepackage{cancel}
\usepackage[space]{cite}
\usepackage{pdfsync}
\usepackage{balance}
\usepackage[ruled]{algorithm2e}
\usepackage{color}
\usepackage{mathtools}
\usepackage{bm}
\usepackage{diagbox}
\usepackage{float}
\usepackage{epstopdf}
\usepackage{pifont}
\usepackage{multirow}
\usepackage{url}
\usepackage{verbatim}
\usepackage{booktabs}
\usepackage{graphicx}
\usepackage{threeparttable}  
\usepackage{makecell}
\usepackage{soul}
\usepackage{xcolor}
\usepackage[normalem]{ulem}
\usepackage{etoolbox}
\usepackage{verbatim} 
\usepackage[linkcolor=black,citecolor=black,urlcolor=black,colorlinks=true]{hyperref}

\DeclareMathOperator*{\argmin}{arg\,min}

\graphicspath{{./figure/}}
\graphicspath{{./figure2/}}
\DeclareGraphicsExtensions{.png,.jpg,.eps,.pdf}

\newcommand{\delete}[1]{{\bgroup\markoverwith{\textcolor{red}{\rule[0.5ex]{2pt}{0.4pt}}}\ULon{#1}}}
\newcommand{\deletefig}[1]{{\bgroup\markoverwith{\textcolor{red}{\rule[2.5ex]{2pt}{2.0pt}}}\ULon{#1}}}

\setstcolor{red}
\hypersetup{draft}

\newtheorem{theorem}{Theorem}
\newtheorem{problem}{Problem}[section]

\title{\LARGE \bf Bearing-based Relative Localization for \\
Robotic Swarm with Partially Mutual Observations}
\begin{document}
	\author{Yingjian Wang, Xiangyong Wen, Yanjun Cao, Chao Xu, Fei Gao
		\thanks{All authors are with the State Key Laboratory of Industrial Control Technology, Zhejiang University, Hangzhou, 310027, China, and the Huzhou Institute of Zhejiang University, Huzhou, 313000, China. Email: \{yj\_wang, fgaoaa\}@zju.edu.cn}
		\thanks{Corresponding author: Fei Gao}}
	\maketitle
	\begin{abstract}
			Mutual localization provides a consensus of reference frame as an essential basis for cooperation in multi-robot systems. Previous works have developed certifiable and robust solvers for relative transformation estimation between each pair of robots. However, recovering relative poses for robotic swarm with partially mutual observations is still an unexploited problem. In this paper, we present a complete algorithm for it with optimality, scalability and robustness. Firstly, we fuse all odometry and bearing measurements in a unified minimization problem among the Stiefel manifold. Furthermore, we relax the original non-convex problem into a semi-definite programming (SDP) problem with a strict tightness guarantee. Then, to hold the exactness in noised cases, we add a convex (linear) rank cost and apply a convex iteration algorithm. We compare our approach with local optimization methods on extensive simulations with different robot amounts under various noise levels to show our global optimality and scalability advantage. Finally, we conduct real-world experiments to show the practicality and robustness.
	\end{abstract} 
	\vspace{-0.4cm}

	\section{Introduction}
	\label{sec:Introduction} 
	
	Recently, robotic swarms have emerged as an upgrading system to single robots since they can be more efficient and fault-tolerant in various complex missions, including cooperative exploration \cite{gao2021meeting}, package delivery \cite{dorling2016vehicle} and surveillance \cite{pasqualetti2012cooperative}. 
In these collaborative tasks, sharing a reference frame among robots is fundamental and essential. 
However, this requirement is hard to meet in GPS-denied environments such as underground caves or indoor rooms. 
Thus, to extend the applicable range of a swarm system, robots need to reach a consensus of coordination using only onboard sensors.
	
	Relative pose estimation, which seek to recover relative transformations between robots, is a fundamental problem in multi-robot systems. 
Bearing-based mutual localization \cite{martinelli2005multi, zhou2012determining, franchi2009mutual, dhiman2013mutual,nguyen2020vision, jang2021multirobot, 9827567} only uses 2D visual detection of robots. 
Compared with map-based localization \cite{schmuck2019ccm, forster2013collaborative, cunningham2013ddf, cieslewski2018data}, it is less influenced by the environment and needs less bandwidth. 
Despite its appeal, there are still some challenges when applying it in real-world multi-robot systems. 
Firstly, as shown in Fig.\ref{fig:head}, due to the limited field-of-view (FOV) and unavoidable occlusion, most robots can only observe a part of other robots, forming a \emph{partial observation graph}. 
How to leverage all measurements in a swarm system with partially mutual observations to recover all robots' poses is still an unexploited problem. 
Moreover, traditional local optimization-based methods \cite{nguyen2020vision, jang2021multirobot}, may potentially fall into local minimum due to the extreme non-convexity of the complicated problem formulation, which is common in relative localization. And they will lead to erroneous solutions, undermining  multi-robot cooperation.

	\begin{figure}[t]
		\centering
		\includegraphics[width=0.5\textwidth]{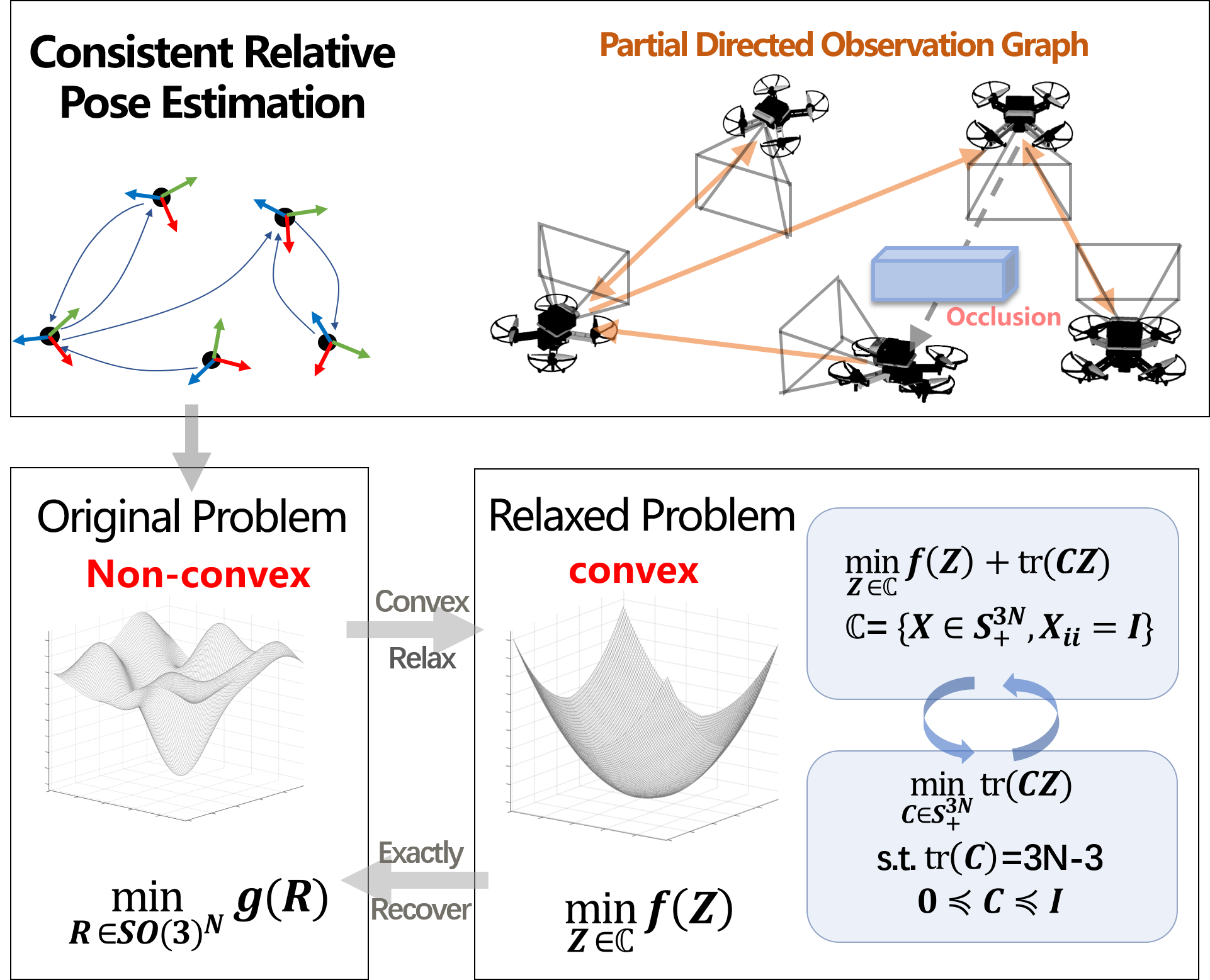}
		\caption{\label{fig:head} An overview of our proposed algorithm in this paper. Based on observation graph build with pose vertices (black dot) and directed edge (blue curves with arrows), our method obtain the global optimal relative pose by convex relaxation, iterative optimization and exact recovery.}
		\vspace{-0.6cm}
	\end{figure}

	We aim to provide a complete relative localization method for a multi-robot system sharing partial mutual observations, which are common for FoV constrained swarm in practice. 
To the best of our knowledge, there is no previous work soving this problem decently. 
In this paper, we focus on a directed graph, which is formed by partial bearing obsercations. 
Firstly we derive two feasible problem formulations based on different variable marginalization methods, resulting in a unified minimization problem among Stiefel manifold \cite{boumal2015riemannian}. 
We also present detailed analysis and comparison between these formulations.
Furthermore, we manage to relax the original non-convex problem into a SDP problem, which is convex and can be minimized globally. 
In addition, we provide a sufficient condition under which we strictly guarantee the tightness of the relaxation in noise-free cases. 
Finally, to prevent the solution's unexactness caused by noises in practice, we borrow the idea of rank-constrained optimization.
More precisely, we apply a fast, minimal-rank-promoting algorithm to our relaxed problem and develop a complete iterative method, as shown in Fig.\ref{fig:head}. 
Extensive experiments on synthetic and real-world datasets show our method's optimality, effectiveness and robustness under different noise levels.
	
	Summarize our contributions in this paper:
	
	\begin{enumerate}
		\item We provide novel and unified formulations for a partially mutual observed multi-robot sysem, to jointly optimize poses of all robots.
		\item We propose a method to relax the original non-convex problem into a SDP problem among a convex set and provide its tightness analysis in noise-free cases.
		\item We combine the rank-constrained convex iteration algorithm with our relaxation method to guarantee its exactness under different noise levels. 
		\item We conduct sufficient simulation and real-world experiments to validate the optimality, practicality and robustness of our proposed method.
		\item We release the implementation of our method in MATLAB and C++ for the reference of our community \footnote{https://github.com/ZJU-FAST-Lab/CertifiableMutualLocalization}.
	\end{enumerate}
	
	\section{Related Work}
	\label{sec:RelatedWork} 
	\subsection{Relative Pose Estimation}
		
		There are mainly two approaches in multi-robot relative pose estimation: map-based method using the loop-closing module and mutual localization using inter-robot observations. Map-based methods, including centralized \cite{schmuck2019ccm, forster2013collaborative} and decentralized \cite{cunningham2013ddf, cieslewski2018data} architectures, typically begin with interloop detection, extract features on common view areas for descriptor matching, and finally construct geometry constraint to recover relative poses. However, this method requires significant bandwidth for feature sharing and has poor performance in environments with less texture or with many similar scenes.
		
		In contrast, mutual localization, which employs robot-to-robot range and bearing measurements for relative pose estimation, relies less on the environment. Martinelli \cite{martinelli2005multi} use the extended Kalman Filter (EKF) as nonlinear estimator for relative localization to fuse mutual measurements. In addition, Zhou \cite{zhou2012determining} provides algebraic and numerical solution methods for any combination of range and bearing measurements. However, since only the minimum number of measurements are considered, it suffers from degeneration under noise.
		
		Bearing measurements are important relative observation resources. The mutual localization using bearing measurements is addressed in  \cite{franchi2009mutual}, which resolves the anonymity of observations with particle filters (PF). Dhiman \cite{dhiman2013mutual} localizes cameras using reciprocal fiducial observations to obviate the assumption of sensor ego motion. In \cite{nguyen2020vision}, the coupled probabilistic data association filter is adopted in mutual localization to reject the false positives/negatives from the vision sensor and IMU. In addition, Jang \cite{jang2021multirobot} proposes an alternating minimization algorithm to optimize relative transformations and scales of local maps in multi-robot monocular SLAM using bearing measurements. These local optimization methods rely on good initial values to obtain a suitable solution. Our previous work \cite{9827567} formulates a mixed-integer programming problem to recover optimal relative poses and data associations jointly with an optimality guarantee. 
		However, all the above works realize mutual localization by relative pose estimation between each pair of robots. In contrast, we directly fuse all team robots' observations to formulate a unified problem and jointly estimate all robots' relative pose, which is meaningful for FOV-limited swarms in environments with obstacles. And it's exactly our focus in this paper.
		
	\subsection{Convex Relaxation in Robotics}
	
	Thanks to advanced optimization theory, a series of solvers for non-convex and NP-hard problems have been developed in computer vision and robotics in the past few years. In \cite{boumal2015riemannian}, Boumal provides a detailed analysis of semidefinite relaxation (SDR), a basic yet powerful convex relaxation method, and proposes the Riemannian staircase algorithm to optimize problems among Stiefel manifold efficiently. 
	Based on the Riemannian staircase algorithm, SE-Sync \cite{rosen2019se} and Cartan-Sync \cite{briales2017cartan} can recover the certifiably optimal solution of pose graph optimization under acceptable noise. 
	In \cite{yang2020teaser}, point registration with outlier is formulated as a quadratically constrained quadratic problem (QCQP) by binary cloning, relaxed by SDR, and globally optimized. 
	Zhao \cite{zhao2020efficient} proposes an efficient SDR-based method for essential matrix estimation, one of the most classical problems in computer vision. 
	Cifuentes \cite{cifuentes2022local} proves the tightness and robustness of SDR  under the assumption of low noise based on local stability theory. 
	In our paper, we also utilize SDR to relax the original non-convex problem and borrow the idea of convex iteration in \cite{dattorro2005convex} to hold exactness under noise. The method of convex iteration is also applied in \cite{giamou2022convex} to develop a distance-geometric inverse kinematic solver. 
	
	\section{Problem Formulation}
	\label{sec:problem_formulation} 
	In this section, we begin with a summary of our notation. Then we provide two problem formulations for relative localization with different approaches to variable elimination. We also provide some analysis of two different formulations.
	
	\subsection{Notation}
	Boldface lower and upper case letters (e.g. x and $\mathbf{M}$) represent vectors and matrices respectivly. Matrices $\mathbf{M} \in \mathbb{R}^{dm \times dm}$ are thought of as block matrices with blocks of size $d \times d$. Subscript indexing such as $\mathbf{M}_{ij}$ refers to the block on the $i$th row and $j$th column of blocks, $1\leq i,j\leq m$. The Kronecker product is writen as $\otimes$. The pseudoinverse and trace of the matrix $\mathbf{M}$ are denoted as $\mathbf{M}^{\dagger}$ and $\text{tr}(\mathbf{M})$ respectively. vec($\cdot$) vectorizes a matrix by stacking its columns on top of each other and $\text{vec}^{-1}(\cdot)$ denotes its inverse. We write $\mathbf{I_n}$ and $\mathbf{0_n}$ (or $\mathbf{I}$ and $\mathbf{0}$ when clear from context) for the $n \times n$ identity matrix and zero matrix respectively. We also write $e_i = [0,\cdots,1,\cdots,0]^T$ as the $i$-th canonical basis vector for the $n$-dimensional space. The space of $n \times n$ symmetric and symmetric positive semi-definite (PSD) matrix are denoted $\mathbb{S}^{n}$ and $\mathbb{S}_+^{n}$, and we write $\mathbf{A} \succeq \mathbf{B}$ to indicate that $\mathbf{A}-\mathbf{B}$ is PSD. 
	
	As the Fig.\ref{fig:head} shown, we model $N$ robots as an directed graph $\mathbf{\mathcal{G}} = (\mathbf{\mathcal{V}}, \mathbf{\mathcal{E}})$, where $\mathcal{V} := \{1,2,...,N\}$ is the set of vertices, and $\mathcal{E} \subset \mathcal{V} \times \mathcal{V}$ is the set of edges. In graph $\mathcal{G}$, the vertex $i$ represent the $i^{th}$ robot with pose in the world $\{\mathbf{R_i}, t_i\}$ and the directed edge $e_{ij} \in \mathcal{E}$ means the robot $i$ can obtain a series of robot $j$'s bearing measurements  $b_{ij}=\{b_1,b_2,...,b_m\}$, where $b_k\in \mathbb{R}^3$ is a unit vector and $m$ is the measurements amount. If there is no observation between one pair of robots in a team, the graph $\mathbf{\mathcal{G}}$ is a partially connected digraph. In the following subsections, we will utilize data from $\mathbf{\mathcal{E}}$ to formulation problems.

	\subsection{Formulation Using Cross Product}
	\label{subsec:formulation1}
	\begin{figure}[t]
		\centering
		\includegraphics[width=0.4\textwidth]{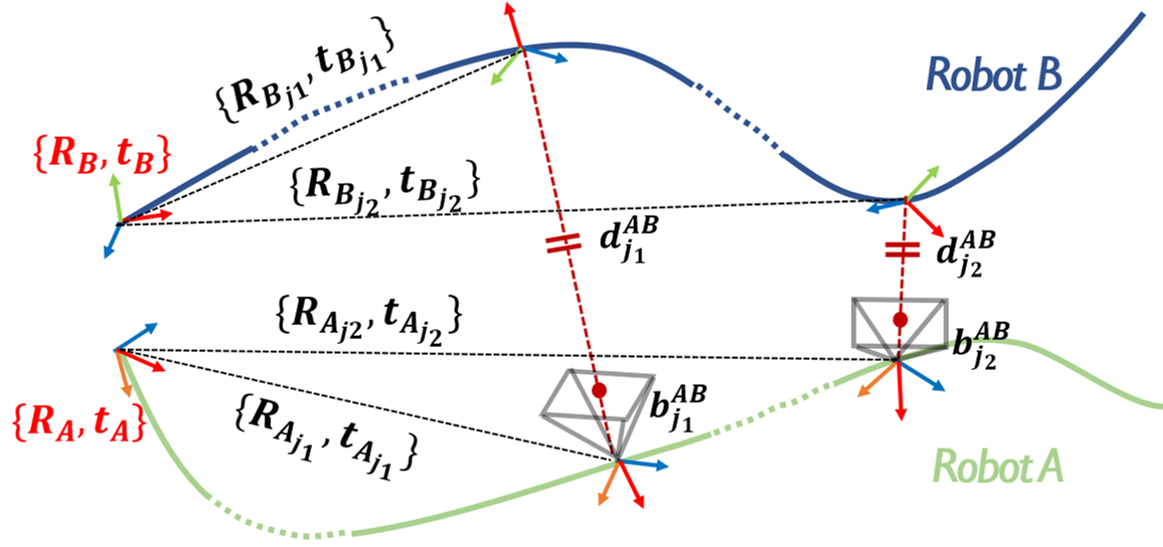}
		\caption{\label{fig:method} Demonstration of the pose in the world $\{\mathbf{R_X}, t_X\}$ and local odometry $\{\mathbf{R_{X_t}},t_{X_t}\}$ of the robot $X$, $X\in\{A,B\}$, and the bearing observations $b^{AB}_t$ and the distance $d^{AB}_t$  between robots at time $t$.}
		\vspace{-0.6cm}
	\end{figure}

	For simplicity, we start with the edge $e_{AB} \in \mathbf{\mathcal{E}}$ between two robots, the observer robot $A$ and the observed robot $B$. As shown in Fig.\ref{fig:method}, consider two timestamps $j_1$ and $j_2$, we have equations as follows:
	\begin{align}
		&\mathbf{R_A} (\mathbf{R_{A_{j_1}}} b^{AB}_{j_1} d^{AB}_{j_1} + t_{A_{j_1}}) + t_{A} = \mathbf{R_B} t_{B_{j_1}} + t_{B}, \\
		&\mathbf{R_A} (\mathbf{R_{A_{j_2}}} b^{AB}_{j_2} d^{AB}_{j_2} + t_{A_{j_2}}) + t_{A} = \mathbf{R_B} t_{B_{j_2}} + t_{B},
	\end{align}
	where $\{\mathbf{R_{X_t}},t_{X_t}\}, b^{AB}_t$ and $d^{AB}_t$ are the local odometry of the robot $X$, the bearing observation and the distance between robots at time $t$ respectively, $X\in\{A,B\}$. Then we eliminate $t_A$ and $t_B$ by subtraction between above equations and obtain 
	\begin{gather}
		\label{equ:imp}
		\mathbf{R_A}  (g_{j_2} d^{AB}_{j_2} -g_{j_1} d^{AB}_{j_1} + t_{A_{j_1j_2}})= \mathbf{R_B}  t_{B_{j_1j_2}}
	\end{gather}
	where $g_t = \mathbf{R_{A}} b^{AB}_t$ for $t \in \{j_1, j_2\}$ and $t_{X_{j_1j_2}} = t_{X_{j_2}} - t_{X_{j_1}}$ for $X \in \{A, B\}$. This step aims to eliminate unconstrained variables (translations) and to formulate a rotation-only problem later, which is a common way to restrict all involved variables in the special orthogonal group ($SO(3)$) \cite{rosen2019se, 9827567}. Then we denote $k_{j_1j_2} = g_{j_1} \times g_{j_2}$ and derive an error
	\begin{gather}
		\label{equ:error}
			e^{AB}_{j_1j_2} = k_{j_1j_2}^T(\mathbf{R_A}^T \mathbf{R_B} t_{B_{j_1j_2}} -t_{A_{j_1j_2}}).
	\end{gather}
	We denote $\mathbf{R}:= [\mathbf{R_1}, \mathbf{R_2},..., \mathbf{R_N}] \in SO(3)^N$ as our decision variable and construct the choose matrix $\mathbf{C_X} = e_X \otimes I_3$ for robot $X$. After substituting $\mathbf{R_X} = \mathbf{R} \mathbf{C_X}$ into Equ. (\ref{equ:error}), we obtain the error involving only variable $\mathbf{R}$:
	\begin{gather}
		\label{error:1}
			e^{AB}_{j_1j_2} = k_{j_1j_2}^T(\mathbf{C_A^T R^T R C_B} t_{B_{j_1j_2}} -t_{A_{j_1j_2}}).
	\end{gather}
	Finally, we collect all measurements, including each robot's local odometry and bearing observations among a set of timestamps $J$, to formulate a least-square problem:
	
	\begin{problem}[Rotation-Only Least-Square Problem]
		\label{pro:rols}
		\begin{equation}
			\begin{split}
				R^* =&\underset{\mathbf{R} \in SO(3)^N}{\arg\min}\ \sum_{e_{XY} \in \mathcal{E} \atop \{j_1,j_2\}\in J} (e^{XY}_{j_1j_2})^T  e^{XY}_{j_1j_2}  \\
				=&\underset{\mathbf{R} \in SO(3)^N}{\arg\min}\ \sum_{e_{XY} \in \mathcal{E}  \atop \{j_1,j_2\}\in J} 
				\begin{aligned} 
					&(\mathbf{C_X^T R^T R C_Y} t_{Y_{j_1j_2}} -t_{X_{j_1j_2}})^T k_{j_1j_2}\\
					&k_{j_1j_2}^T(\mathbf{C_X^T R^T R C_Y} t_{Y_{j_1j_2}} -t_{X_{j_1j_2}})
				\end{aligned} \\
				=&\underset{\mathbf{R} \in SO(3)^N}{\arg\min}\ \sum_{e_{XY} \in \mathcal{E}  \atop \{j_1,j_2\}\in J} 			
				\begin{aligned} 
					& \textup{tr}(\mathbf{A^{XY}_{j_1j_2} R^TR B^{XY}_{j_1j_2} R^TR}) \\
					& +\textup{tr}(\mathbf{C^{XY}_{j_1j_2} R^TR}) + D^{XY}_{j_1j_2} \\
				\end{aligned}  \\
				&\ \scriptsize \color{gray}(\text{$\textup{tr}(\mathbf{AXBX}) = \textup{vec}(\mathbf{X})^T (\mathbf{B} \otimes \mathbf{A}) \textup{vec}(\mathbf{X})$}) \\
				=&\underset{\mathbf{R} \in SO(3)^N}{\arg\min} \textup{vec}(\mathbf{R^TR})^T \mathbf{H}  \textup{vec}(\mathbf{R^TR}) + \textup{tr}(\mathbf{J R^TR}) + K, \nonumber
			\end{split}
		\end{equation}
	\end{problem}
where 
	\begin{align}
		& \mathbf{A^{XY}_{j_1j_2}} = \mathbf{C_Y} t_{Y_{j_1j_2}} t_{Y_{j_1j_2}}^T \mathbf{C_Y^T} \succeq 0, \\
		& \mathbf{B^{XY}_{j_1j_2}} = \mathbf{C_X} k_{j_1j_2} k_{j_1j_2}^T \mathbf{C_X^T} \succeq 0, \\
		& \mathbf{C^{XY}_{j_1j_2}} = -2 \mathbf{C_Y} t_{Y_{j_1j_2}} t_{X_{j_1j_2}}^T k_{j_1j_2} k_{j_1j_2}^T \mathbf{C_X^T} , \\
		& D^{XY}_{j_1j_2} = t_{X_{j_1j_2}}^T k_{j_1j_2} k_{j_1j_2}^T t_{X_{j_1j_2}}, \\
		& \mathbf{H} = \sum_{e_{XY} \in \mathcal{E} \atop \{j_1,j_2\}\in J}
		\mathbf{B^{XY}_{j_1j_2}} \otimes  \mathbf{A^{XY}_{j_1j_2}} \ \in \mathbb{R}^{9N^2 \times 9N^2}, \\
		& \mathbf{J} = \sum_{e_{XY} \in \mathcal{E} \atop \{j_1,j_2\}\in J} \mathbf{C^{XY}_{j_1j_2}},  \qquad
		K = \sum_{e_{XY} \in \mathcal{E} \atop \{j_1,j_2\}\in J} D^{XY}_{j_1j_2}.  
	\end{align}

	Note that  $A^{XY}_{j_1j_2}$ and $B^{XY}_{j_1j_2}$ are PSD since they are Gram matrices. Then, as that $\mathbf{A}  \otimes \mathbf{B}  \succeq 0$ if $\mathbf{A}  \succeq 0$ and $\mathbf{B}  \succeq 0$, we obtain that $\mathbf{H}  \succeq 0$.
	
	\subsection{Problem Formulation Using Schur Complement}
	\label{subsec:formulation2}
	In this subsection, we provide another problem formulation, which takes relative rotations and distances between robots as variables and marginalizes distances using Schur Complement. More specifically, we first derive an error according to Equ. (\ref{equ:imp}) via decision variable $\mathbf{R}$ and choose matrix $\mathbf{C_i}$ as follows:  
	\begin{gather}
		\label{error:2}
		e_{j_1j_2}^{AB} = g_{j_2} d^{AB}_{j_2} -g_{j_1} d^{AB}_{j_1} + t_{A_{j_1j_2}} - \mathbf{C_A^T R^T R C_B} t_{B_{j_1j_2}}. 
	\end{gather}
	Next we define some variables:
	\begin{align}
		d^{XY} &\triangleq \text{vstack}(\{d^{XY}_{j}\}_{j \in J}), \\
		d      &\triangleq \text{vstack}(\{d^{XY}\}_{e_{XY} \in \mathcal{E}}), \\
		x 	   &\triangleq [\text{vec}(R^TR)^T,\  y,\ d^T]^T
	\end{align}
	with auxiliary constraint $y^2=1$, where the notation $\text{vstack}(G)$ stacks all variable in $G$ vertically. Using above variables, we rewrite Equ.\ref{error:2} in linear form $e^{AB}_{j_1j_2} = \mathbf{M^{AB}_{j_1j_2}} x$, where 
	\begin{gather}
		\mathbf{M^{AB}_{j_1j_2}} \triangleq [\mathbf{M_1}, M_2, \mathbf{M_3}], \\
		\mathbf{M_1} = \mathbf{C_A^T} ((\mathbf{C_B} t_{B_{j_1j_2}})^T \otimes I ), \quad
		M_2 = t_{A_{j_1j_2}}, \\
		\mathbf{M_3} = [0 \cdots -g_{j_1} \cdots g_{j_2} \cdots 0].
	\end{gather}
	Two non-zero elements of $M_3$ are columns corresponding to the distance variables between two robots at time $j_1$ and $j_2$.
	
	Next, we use all errors among $\mathcal{G}$ to formulate a least-square problem involving rotation and distance variables:
	\begin{problem}[Non-Marginalized Problem]
		\label{pro:origin}
		\begin{equation}
			\begin{split}
				x^* &=\underset{x}{\arg\min}\ \sum_{e_{XY} \in \mathcal{E} \atop \{j_1,j_2\}\in J} (e^{XY}_{j_1j_2})^T  e^{XY}_{j_1j_2}  \\
				&=\underset{x}{\arg\min}\ \sum_{e_{XY} \in \mathcal{E} \atop \{j_1,j_2\}\in J}x^T \mathbf{M^{XY}_{j_1j_2}}^T \mathbf{M^{XY}_{j_1j_2}} x \\
				&=\underset{x}{\arg\min}\ x^T \underbrace{(\sum_{e_{XY} \in \mathcal{E} \atop \{j_1,j_2\}\in J} \mathbf{M^{XY}_{j_1j_2}}^T \mathbf{M^{XY}_{j_1j_2}})}_{:=\mathbf{Q}} x \\
				&\textup{s.t} \quad  y^2=1, \quad \mathbf{R_i} \in SO(3) \quad \forall i \in [1,N]. \notag
			\end{split}
		\end{equation}
	\end{problem}
	
	To marginalize distance variables $d$ which is constraint-less, we follow the same procedure in \cite{9827567} using Schur Complement. We write $\mathbf{Q}$ as follows:
	\begin{gather}
		\mathbf{Q} = 
		\begin{bmatrix}
			\mathbf{Q_{\bar{\mathcal{D}},\bar{\mathcal{D}}}} & \mathbf{Q_{\bar{\mathcal{D}},\mathcal{D}}} \\
			\mathbf{Q_{\mathcal{D},\bar{\mathcal{D}}}} & \mathbf{Q_{\mathcal{D},\mathcal{D}}}
		 \end{bmatrix},
	\end{gather}
	where the subindex $\mathcal{D}$ stands for the set of indexes corresponding to distance variables or not (subindex $\bar{\mathcal{D}}$). Furthermore, we eliminate  $d$ to obtain the following problem:
	
	\begin{problem}[Marginalized Problem]
		\label{pro:marg}
		\begin{gather}
			z^* =\underset{z}{\arg\min}\ z^T \mathbf{\bar{Q}} z \notag \\
			\textup{s.t} \  y^2=1, \ \mathbf{R_i} \in SO(3) \quad \forall i \in [1,N]. \notag
		\end{gather}
	\end{problem}
	where $z = [\text{vec}(\mathbf{R^TR})^T,\ y]^T$ and $\mathbf{\bar{Q}} = \mathbf{Q} / \mathbf{Q_{\mathcal{D},\mathcal{D}}} = \mathbf{Q_{\bar{\mathcal{D}},\bar{\mathcal{D}}}} - \mathbf{Q_{\bar{\mathcal{D}},\mathcal{D}}} \mathbf{Q_{\mathcal{D},\mathcal{D}}^{\dagger}} \mathbf{Q_{\mathcal{D},\bar{\mathcal{D}}}} \in \mathbb{R}^{(9N^2+1) \times (9N^2+1)}$. We take $y = 1$ and rewrite Problem \ref{pro:marg} as a rotation-only problem:
	\begin{problem}[Rotation-Only Marginalized Problem]
		\label{pro:romp}
			\begin{gather}
				\mathbf{R}^*=\underset{\mathbf{R} \in SO(3)^N}{\arg\min} \textup{vec}(\mathbf{R^TR})^T \mathbf{U} \textup{vec}(\mathbf{R^TR}) + \textup{tr}(\mathbf{V} \mathbf{R^TR}) + W,  \notag
			\end{gather}
	\end{problem}
	where
	\begin{gather}
		\mathbf{U} = \textup{Block}(\mathbf{\bar{Q}}, 9N^2, 9N^2, 1, 1), \\
		\mathbf{V} = 2\textup{vec}^{-1}(\textup{Block}(\mathbf{\bar{Q}},9N^2, 1, 1, 9N^2+1)),  \\
		W = \textup{Block}(\mathbf{\bar{Q}}, 1, 1, 9N^2+1, 9N^2+1). 
	\end{gather}
	Block$(\mathbf{X}, c, r, x, y)$ denotes ($c \times r$)-block of $\mathbf{X}$ from the index ($x,y$).
	Note that if Gram matrix $\mathbf{Q} \succeq 0$, then $\mathbf{\bar{Q}}$ is still PSD after Schur Complement, as well as its $(9N^2)$th leading principal submatrix $\mathbf{U}$.
	
	\subsection{Analysis}
	We conclude the constructed problem with two different formulations as a unified problem:
	\begin{problem}[Unified Problem]
		\label{pro:up}
			\begin{gather}
				\underset{\mathbf{R} \in SO(3)^N}{\min} g(\mathbf{R}) = \textup{vec}(\mathbf{R^TR})^T \mathbf{A} \textup{vec}(\mathbf{R^TR}) + \textup{tr}(\mathbf{B R^TR}) + C. \notag
			\end{gather}
	\end{problem}
	Firstly, note that whatever formulation we choose, only rotation-related variables are left, and the number of them is solely related to $N$. In contrast, the variable amount in local optimization methods \cite{jang2021multirobot} is always related to the measurement amount $k$. Compared with them, our method needs less computation and memory. 
	
	Then we study different error modeling behind two formulations. As the Fig.\ref{fig:error} shown, we firstly transform $t_{B_{j_1j_2}}$ to robot $A$'s coordinate frame by the truth (latent) relative rotation $\mathbf{\tilde{R}_{AB}}$ and introduce $\Delta$ as the observation error due to noise in visual detection. That means implicitly, we have
	\begin{gather}
		\mathbf{\tilde{R}_{AB}} t_{B_{j_1j_2}} + d^{AB}_{j_1} b^{AB}_{j_1} + \Delta =t_{A_{j_1j_2}} + d^{AB}_{j_2} b^{AB}_{j_2}. 
	\end{gather}
	Thus, it is apparent that we aim to minimize $(k_{j_1j_2}^T \Delta)^2$  and $\Vert \Delta \Vert$ in the first and second formulations, respectively.
	
	\begin{figure}[t]
		\centering
		\includegraphics[width=0.5\textwidth]{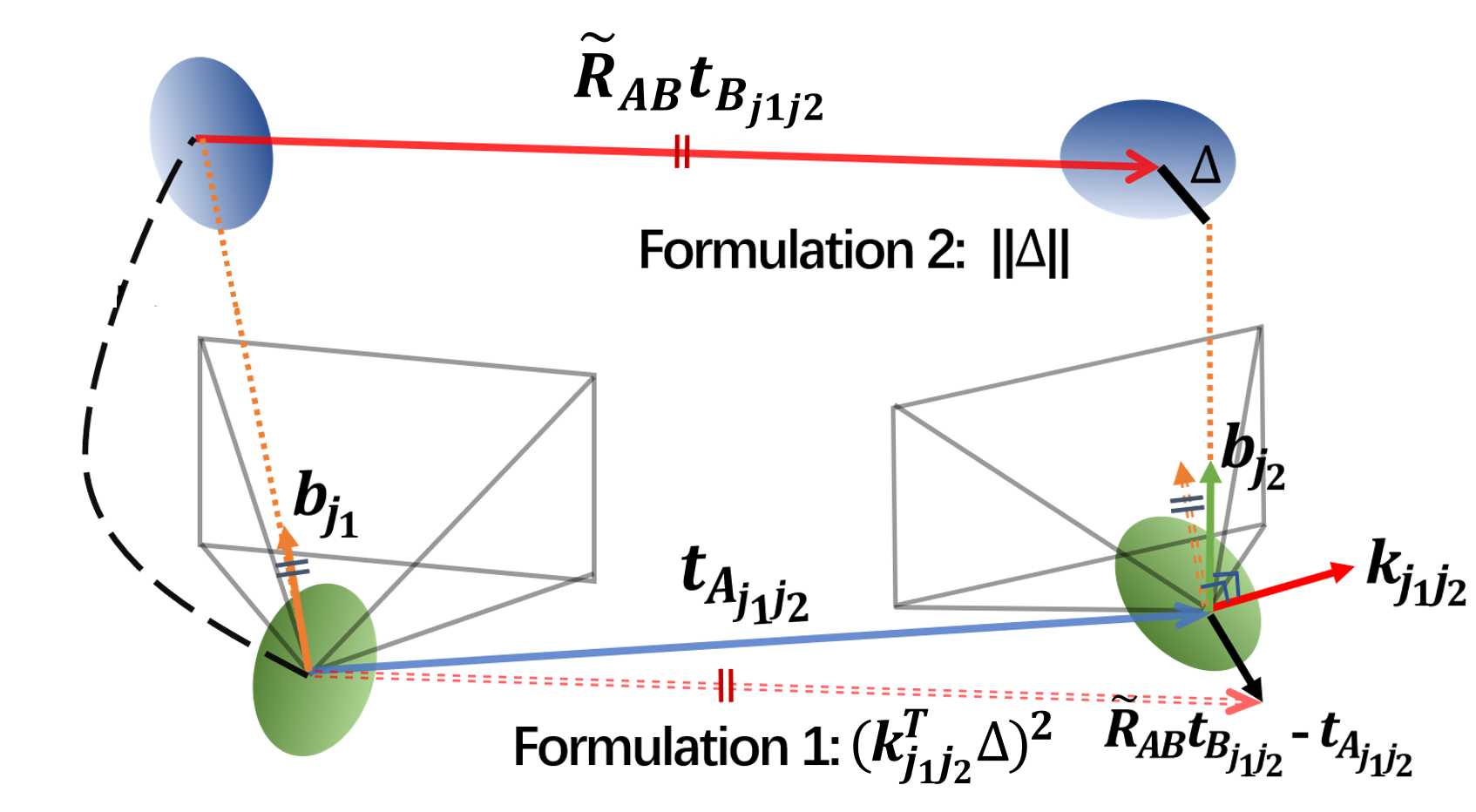}
		\caption{\label{fig:error} Demonstration of different error modeling in two formulations.}
		\vspace{-0.3cm}
	\end{figure}

	Although there are no apparent differences in error modeling, we recommend the first formulation. That is because the second method needs not only large-scale matrix multiplication (the scale depends on the measurement amount) but also dense pseudoinverse calculation in Schur Compliment, leading to more memory usage than another formulation. Besides, the sparsity of $\mathbf{H}$ in the first formulation can be utilized to accelerate optimization significantly.
	
	Then, we analyze the convexity of the formulated problem. $g(\mathbf{R})$ is a four-term non-convex function among $SO(3)^N$, which renders Problem \ref{pro:up} tough to be minimized globally. The local optimization-based method \cite{jang2021multirobot} typically seeks good initial guesses by introducing more or fewer assumptions when facing non-convex problems. In contrast, in the next section, we aim to relax Problem \ref{pro:up}  into a convex problem and recover the globally optimal solution exactly.

	\section{Convex Iterative for Consistent Relative Pose Estimation}
	In this section, we will apply semidefinite relaxation
	to Problem \ref{pro:up}  and obtain the optimal solution in Sec.\ref{subsec:SDR}. Then, we provide a condition under which tightness of relaxation can be guaranteed in noise-free cases with strict proof in Sec.\ref{subsec:tightness}. Lastly, to ensure exactness in noised-cases, we introduce rank cost and iterative algorithm to our relaxed problem in Sec.\ref{subsec:convex}.
	
	\subsection{Semidefinite Relaxation and Lagrangian Dual}
	\label{subsec:SDR}
	To relax Problem \ref{pro:up}, we first consider introducing a new matrix variable $\mathbf{Z} \succeq 0$ to replace the original variable $\mathbf{R^TR}$. Then we drop the non-convex constraint $\textup{rank}(\mathbf{Z}) = 3$ to obtain
 	\begin{problem}[Relaxed Problem]
 		\label{pro:rp}
 		\begin{equation}
 			\begin{aligned}
 				\underset{\mathbf{Z} \in \mathcal{C}}{\min} f(\mathbf{Z}) = \textup{vec}(\mathbf{Z})^T \mathbf{A} \textup{vec}(\mathbf{Z}) + \textup{tr}(\mathbf{B} \mathbf{Z}) + C, \\
 				\mathcal{C} = \{ \mathbf{X} \in \mathbb{S}_{+}^{3N}: \mathbf{X} \succeq 0, \mathbf{X_{ii}} = \mathbf{I_3} ,i \in [1,N]\} \notag
 			\end{aligned}
 		\end{equation}
 	\end{problem}
 	The Lagrangian corresponding to 
 	Note that $\mathcal{C}$ is a compact convex set and $f(\mathbf{Z})$ is a quadratic function after variable substitution. Thus if $\mathbf{A} \succeq 0$, the Problem \ref{pro:rp} is a convex SDP problem, which can be  solved globally by off-shelf SDP solvers using interior point method. The Lagrangian dual corresponding to Problem \ref{pro:rp} is 
 	\begin{equation}
 		\begin{aligned}
 			\mathcal{L}& :  \ \mathbb{S}_{+}^{3N} \times \mathbb{S}^{3N} \to \mathbb{R}.\\
 			\mathcal{L}(\mathbf{Z},\mathbf{\Lambda})& = f(\mathbf{Z}) + \sum_{i=1}^{N} \textup{tr}(\mathbf{\Lambda_i}(\mathbf{Z_{ii}} - \mathbf{I_3})) \\
 			& = f(\mathbf{Z}) + \sum_{i=1}^{N} \textup{tr}(\mathbf{\Lambda_i} \mathbf{Z_{ii}})-\textup{tr}(\mathbf{\Lambda_i}) \\
 			& = f(\mathbf{Z}) + \textup{tr}(\mathbf{\Lambda} \mathbf{Z}) - \textup{tr}(\mathbf{\Lambda}).
 		\end{aligned}
 	\end{equation} 
 	where $\mathbf{\Lambda_i} \in  \mathbb{S}^{3}$ and $\mathbf{\Lambda} \triangleq \textup{Diag}(\mathbf{\Lambda_1}, \cdots, \mathbf{\Lambda_N}) \in \mathbb{S}^{3N}$.
 	
 	Ideally, if the solution $\mathbf{Z^*}$ of Problem \ref{pro:rp} meets the condition $\textup{rank}(\mathbf{Z^*}) = 3$ (we will prove it in Sec. \ref{subsec:tightness} under zero noise), we can exactly recover the optimal relative poses $\{\mathbf{R^*}, t^*\}$. Firstly, we deploy a rank-3 decomposition of $\mathbf{Z^*}$ to obtain $\mathbf{Y^*} \in O(3)^N$. Then, for each $\mathbf{Y_i} \in O(3)$, we compute the sigular value decomposition $\mathbf{Y_i^*} = \mathbf{U_i} \mathbf{\Xi_i} \mathbf{V_i^T}$. The rotation we are looking for is then  
 		\begin{align}
 			\mathbf{R_i^*} = \mathbf{V_i} \begin{pmatrix}
 				1 & \ & \\
 				\ & \ddots & \\
 				\ & \      & \textup{det}(\mathbf{V_i U_i^T}) \\
 			\end{pmatrix}  \mathbf{U_i^T}
 		\end{align}
  Finally, distance $d^*$ and relative translations  $\mathbf{t^*} \triangleq [t_1, \cdots, t_N]$ can be recovered respectively in closed- form using $\mathbf{R^*}$.
  
  \subsection{Tightness of Convex Relaxation}
  \label{subsec:tightness}
  In this subsection, we provide proof of the tightness of relaxation in noise-free cases. That means we can exactly recover Problem \ref{pro:up}'s solution $\mathbf{R^*}$ from Problem \ref{pro:rp}'s  solution $\mathbf{Z^*}$, which bases on the following theorems:
  \begin{theorem}
	  \label{the:1}
	  If the solution of Problem \ref{pro:rp} $\mathbf{Z^*}$ can factor as $\mathbf{Z^*} = \mathbf{{R^*}^TR^*}, \mathbf{R^*} \in SO(3)^N$, then $\mathbf{R^*} $ is a global minimizer of Problem \ref{pro:up}.
  \end{theorem}
 \begin{proof}
	Since our relaxation can be seen as the expansion of feasible variable space, the optimal value of $f(\mathbf{Z})$ and $g(\mathbf{R})$ satisfy $f^* \leq g^*$. But if $\mathbf{Z^*}$ admits the factorization, then $\mathbf{R^*}$ is a feasible point of the Problem \ref{pro:up}, and then we have $g^* \leq g(\mathbf{R^*}) = f(\mathbf{Z^*})$. This means $g(\mathbf{R^*}) = g^*$ and consequently that $\mathbf{R^*}$ is a global minimizer of Problem \ref{pro:up}. 
\end{proof}

  \begin{theorem}
  	\label{the:0}
  	We denotes $\mathbf{P} \in \mathcal{P}$ where $\mathcal{P}= \{\mathbf{X} \in \mathbb{S}^{3N}: \mathbf{X} \neq \mathbf{0_{3N}}, \mathbf{X_{ii}} = \mathbf{0_3}, i\in[1,N]\}$. Then, for each $\mathbf{Z} \in \mathbb{S}_{+}^{3N}$ which is constructed by $\mathbf{Z} = \mathbf{R^T} \mathbf{R}$, $\mathbf{R} \in SO(3)^N$, if for any $i\in[1,N]$, there exist $j, j\neq i$  that $\mathbf{P_{ij}} = \mathbf{0_3}$ , then $\mathbf{Z+P}$ is not semi-definite.
  \end{theorem}
	We refer readers to the supplementary material for the proof of Theorem \ref{the:0}. Then based on the above two theorems, we propose the following theorem:
	\begin{theorem}
	\label{the:2}
	 	In noised-free cases, let $\tilde{\mathbf{R}}$ denotes the true (latent) relative rotations. If there is no isolated vertex in  $\mathcal{G}$, there's one and there's only one solution of the relaxed Problem \ref{pro:rp}, and that is $\tilde{\mathbf{Z}} = \tilde{\mathbf{R}}^T \tilde{\mathbf{R}}$.
	\end{theorem}
	\begin{proof}
	The proof of Theorem \ref{the:2} includes two steps: 1. $\tilde{\mathbf{Z}}$ is a solution of Problem \ref{pro:rp}, 2. there is no other solution. We start with the first step. Since both Problem \ref{pro:rols} and Problme \ref{pro:origin} are least-square problems, we can write $g(\mathbf{R}) = e(\mathbf{R})^Te(\mathbf{R})$ and $f(\mathbf{Z}) = v(\mathbf{Z})^T v(\mathbf{Z})$. According to the error modeling in Equ.\ref{error:1} and Equ.\ref{error:2}, we obtain $e(\tilde{\mathbf{R}}) = v(\tilde{\mathbf{Z}}) = 0$, consequently $\nabla g(\tilde{\mathbf{R}}) = 2 e(\tilde{\mathbf{R}}) \nabla e(\tilde{\mathbf{R}}) = \mathbf{0}$ and similarly $\nabla f(\tilde{\mathbf{Z}}) = \mathbf{0}$. 
	Then we let $\tilde{\mathbf{\Lambda}} = 0$, and obtain that $\tilde{\mathbf{Z}}$ is a global minimizer since the optimality condition is satisfied: 1. Primal Feasibility: $\tilde{\mathbf{Z}} \succeq 0$, 2. Dual Feasibility: $\mathbf{Q}(\tilde{\mathbf{Z}}, \tilde{\mathbf{\Lambda}}) = \nabla \mathcal{L}(\tilde{\mathbf{Z}}, \tilde{\mathbf{\Lambda}}) = \nabla f(\tilde{\mathbf{Z}}) + \tilde{\mathbf{\Lambda}} = \mathbf{0} \succeq 0$, 3. Lagrangian multiplier: $\textup{tr}(\mathbf{Q}(\tilde{\mathbf{Z}}, \tilde{\mathbf{\Lambda}}) \tilde{\mathbf{Z}}) = 0$.
	
	Then we prove the second step by contradiction. If there is another global minimizer $\mathbf{G} = \tilde{\mathbf{Z}} + \mathbf{P}$, then $\mathbf{P} \in \mathcal{P}$ and $f(\mathbf{G}) = 0$, which means:
	\begin{equation}
		\begin{aligned}
		f(\mathbf{G}) = &\textup{vec}(\mathbf{G})^T \mathbf{A} \textup{vec}(\mathbf{G}) + \textup{tr}(\mathbf{B} \mathbf{G}) + C \\
		= &\underbrace{\textup{vec}(\tilde{\mathbf{Z}})^T \mathbf{A} \text{vec}(\tilde{\mathbf{Z}}) + \textup{tr}(\mathbf{B} \tilde{\mathbf{Z}}) + C}_{f(\tilde{\mathbf{Z}}) = 0} + \\
		& \textup{vec}(\mathbf{P})^T \mathbf{A} \text{vec}(\mathbf{P}) + (\underbrace{2\textup{vec}(\tilde{\mathbf{Z}}^T) \mathbf{A} + \textup{vec}(\mathbf{B})^T}_{\nabla f(\tilde{\mathbf{Z}}) = \mathbf{0}}) \textup{vec}(\mathbf{P}) \\
		= &\textup{vec}(\mathbf{P})^T \mathbf{A} \textup{vec}(\mathbf{P}) = 0. \nonumber
		\end{aligned}
	\end{equation}
	For simplicity, we only base on the construction of $\mathbf{A}$ ($\mathbf{H}$) in the first formulation to continue derivation.
	\begin{equation}
	\begin{aligned}
	0 = &\textup{vec}(\mathbf{P})^T \mathbf{A} \textup{vec}(\mathbf{P}) \\
	  = &\sum_{e_{AB} \in \mathcal{E} \atop \{j_1,j_2\}\in J} \textup{tr}( \mathbf{A^{XY}_{j_1j_2} P B^{XY}_{j_1j_2} P}) \\
	  = &\sum_{e_{AB} \in \mathcal{E} \atop \{j_1,j_2\}\in J} \textup{tr}(\mathbf{C_B t_{B_{j_1j_2}} t_{B_{j_1j_2}}^T C_B^T  P C_A k_{j_1j_2} k_{j_1j_2}^T C_A^T P}) \\
	  = &\sum_{e_{AB} \in \mathcal{E} \atop \{j_1,j_2\}\in J} (k_{j_1j_2}^T \mathbf{C_A^T P C_B} t_{B_{j_1j_2}} )^2. \nonumber\\
	\end{aligned}
	\end{equation}
	As $k_{j_1j_2}$ and $t_{B_{j_1j_2}}$ are not always zero vector, there must be  $\mathbf{C_A^T P C_B} = \textup{Block}(\mathbf{P},3,3,3A,3B) = \mathbf{P_{AB}} = \mathbf{0}$ for each $e_{AB}$. Then for $\mathcal{G}$ which includes no isolated vertex, it means that for each colunm of $\mathbf{P}$, there at least exsit $\{i,j\}, i\neq j$ that $\mathbf{P_{ij}} = \mathbf{0}$. Then according to the Theorem \ref{the:0}, $\tilde{\mathbf{Z}} + \mathbf{P}$ is not semi-definite, which means there is no another feasible point $\mathbf{G}$ satisfying $f(\mathbf{G}) = f^* = 0$.
	\end{proof}
 	
 	Here we have provided complete proof of tightness for noise-free cases. Unfortunately, in experiments, we find any noise will influence the tightness, resulting in $\textup{rank}(\mathbf{Z^*}) \textgreater 3$. It means that in the real world, noises from detection and odometry drift will greatly affect the exactness of the rank-3 decomposition and lead to enormous errors. Thus, we focus on constraining rank($\mathbf{Z^*}$) for the Problem \ref{pro:rp} under noise.	
 	
 	\subsection{Convex Iteration with Linear Rank Cost}
 	\label{subsec:convex}
	Here, we firstly consider adding a rank penalty to the original cost function to constrain $\textup{rank}(\mathbf{Z^*})$
	\begin{problem}[Problem with Rank Cost]
		\label{pro:rank_cost}
			\begin{gather}
			\underset{\mathbf{Z} \in \mathcal{C}}{\min} \ \textup{vec}(\mathbf{Z})^T A \textup{vec}(\mathbf{Z}) + \textup{tr}(\mathbf{B Z}) + C + \alpha h(\mathbf{Z}) \notag \\
			\mathcal{C} = \{ \mathbf{X} \in \mathbb{S}_{+}^{3N}: \mathbf{X} \succeq 0, \mathbf{X_{ii}} = \mathbf{I_3} ,i \in [1,N]\} \notag 
			\end{gather}
	\end{problem}
	where $h(\mathbf{Z})$ is a rank cost and $\alpha$ is a weight, e.g. $h(\mathbf{Z}) = 	\left| \textup{rank}(\mathbf{Z}) - 3 \right|$. Actually, we can drop $\left|  \cdot \right| $ by assuming that $\textup{rank}(\mathbf{Z^*})$ is always greater than 3 because of little-constrained solution space $\mathcal{C}$. Even so, the cost of rank is non-convex and still challenging  to minimize globally. Thus, we seek for convex (linear) heuristic cost function to replace function $\textup{rank}(\mathbf{Z})$. We consider the following function:
	\begin{equation}
		\begin{aligned}
		h(\mathbf{Z}) = \sum_{i=4}^{3N} \lambda_i(\mathbf{Z}),
 		\end{aligned}
	\end{equation}
	where $\lambda_i(\mathbf{Z})$ is the $i$th largest eigenvalue of $\mathbf{Z}$. Since $\mathbf{Z} \succeq 0$, the minimal value of $h(\mathbf{Z})$ may be zero, which implies that $\textup{rank}(\mathbf{Z}) \leq 3$. Thus, combining it with the assumption $\text{rank}(\mathbf{Z^*}) \geq 3$, we can implicitly constrain $\textup{rank}(\mathbf{Z^*})$. We compute $h(\mathbf{Z})$ by solving the following SDP problem \cite{dattorro2005convex}:

	\begin{problem}[Sum of Unnecessary Eigenvalues]
		\label{pro:trace}
			\begin{gather}
			\sum_{i=4}^{3N} \lambda_i(\mathbf{Z}) = \argmin_{\mathbf{C}\in S^{3N}} \textup{tr}(\mathbf{CZ}) \notag \\
			\textup{s.t}  \ \textup{tr}(\mathbf{C}) = 3N-3,  \quad  \mathbf{0} \preceq \mathbf{C} \preceq \mathbf{I}, \notag
			\end{gather}
	\end{problem}
	which can be solved in closed-form:
	
	\begin{equation}
		\label{equ:close}
		\begin{aligned}
		&\mathbf{C^*} =  \mathbf{O^T} \mathbf{O}, \\
		&\mathbf{O} =  \mathbf{N}(:, 4:3N),
		\end{aligned}
	\end{equation}
	where $\mathbf{V}$ is from the eigendecomposition $\mathbf{Z} = \mathbf{N}^T \mathbf{\Lambda} \mathbf{N}$.
	Finally, combining the rank-constrained SDP and the closed-form calculation of $\mathbf{C}$, we propose an algorithm that optimizes Problem \ref{pro:rank_cost} and Problem \ref{pro:trace} iteratively, summarized in Algorithm \ref{alg:CICRPE}. The weight $\alpha$ is adjusted to tradeoff $f(\mathbf{Z})$ and $h(\mathbf{Z})$. Although there is no strict guarantee about the number of iterations, in practice, $k = 2 \sim 4$ suffices. In actual, the method of convex iteration has been successfully applied to sensor networking localization \cite{dattorro2005convex} and distance-geometric inverse kinematics\cite{giamou2022convex}. And we are the first to introduce it into relative pose estimation. The Fig.\ref{fig:funciton} demonstrates the effect of rank cost in noised cases. Although $g(\mathbf{R})$  is a non-convex function, our method (minimize $f(\mathbf{Z}) + \alpha h(\mathbf{Z})$) can obtain the global minimizer while optimizing pure $f(\mathbf{Z})$ can't due to the information loss in low-rank decomposition.

		\begin{algorithm}[h]  
			\caption{\label{alg:CICRPE}Convex Iteration for Mutual Localization}
			\KwIn{specification in Problem \ref{pro:rank_cost} including $\mathbf{A, B}, C$}
			\KwOut{$\mathbf{Z^*}$ with rank($\mathbf{Z^*}$) = 3}
			Initialize $\mathbf{C^{\{0\}}} = \mathbf{0}$, $\alpha^{\{0\}} = 0$ \\
			\While{\textup{rank}($\mathbf{Z^{\{k\}}}$) $\neq 3$}
			{
				Obtain $\mathbf{Z^{\{k+1\}}}$ by solving Problem \ref{pro:rank_cost} \\
				\quad with $\mathbf{C^{\{k\}}}$ and  $\alpha^{\{k\}}$ \nonumber \\
				Obtain $\mathbf{C^{\{k+1\}}}$ by solving Problem \ref{pro:trace} \\ 
				\quad with $\mathbf{Z^{\{k+1\}}}$ using Equ. (\ref{equ:close}) \\
				Obtain $\alpha^{\{k+1\}} = f(\mathbf{Z^{\{k+1\}}}) / h(\mathbf{Z^{\{k+1\}}})$ \\
			}
			Return $\mathbf{Z^*} = \mathbf{Z^{\{k\}}}$
		\end{algorithm}
	\vspace{-0.3cm}

	\section{Experiment}
	\label{sec:experiment} 
	In this section, we first confirm the optimality of our proposed algorithm by comparing it with local optimization-based methods. In addition, we show our method's scalability with different robots amount. Next, to present the robustness of our method, we compare its performances under different noise levels.	Then, we apply our method in the real world to verify its practicality and robustness. Our algorithm is implemented in MATLAB using cvx \cite{grant2014cvx} and in C++ using MOSEK \cite{mosek} respectively. We run simulated experiments in PC (Intel i5-9400F) and real-world experiments in NUC (i7-8550U) with a single core.
	
	\subsection{Experiments on Synthetic Data}
		To simulate bearing measurement, we follow our previous work \cite{9827567} to generate synthetic data. Firstly we obtain random trajectories for multiple robots. Then we produce noised bearing measurements by adding Gaussian noise with different standard deviations $\sigma$. Finally, we take the first pose of each trajectory as the local world frame to obtain each robot’s local odometry, which is shared with other robots.
	\begin{figure}[t]
		\centering
		\includegraphics[width=0.45\textwidth]{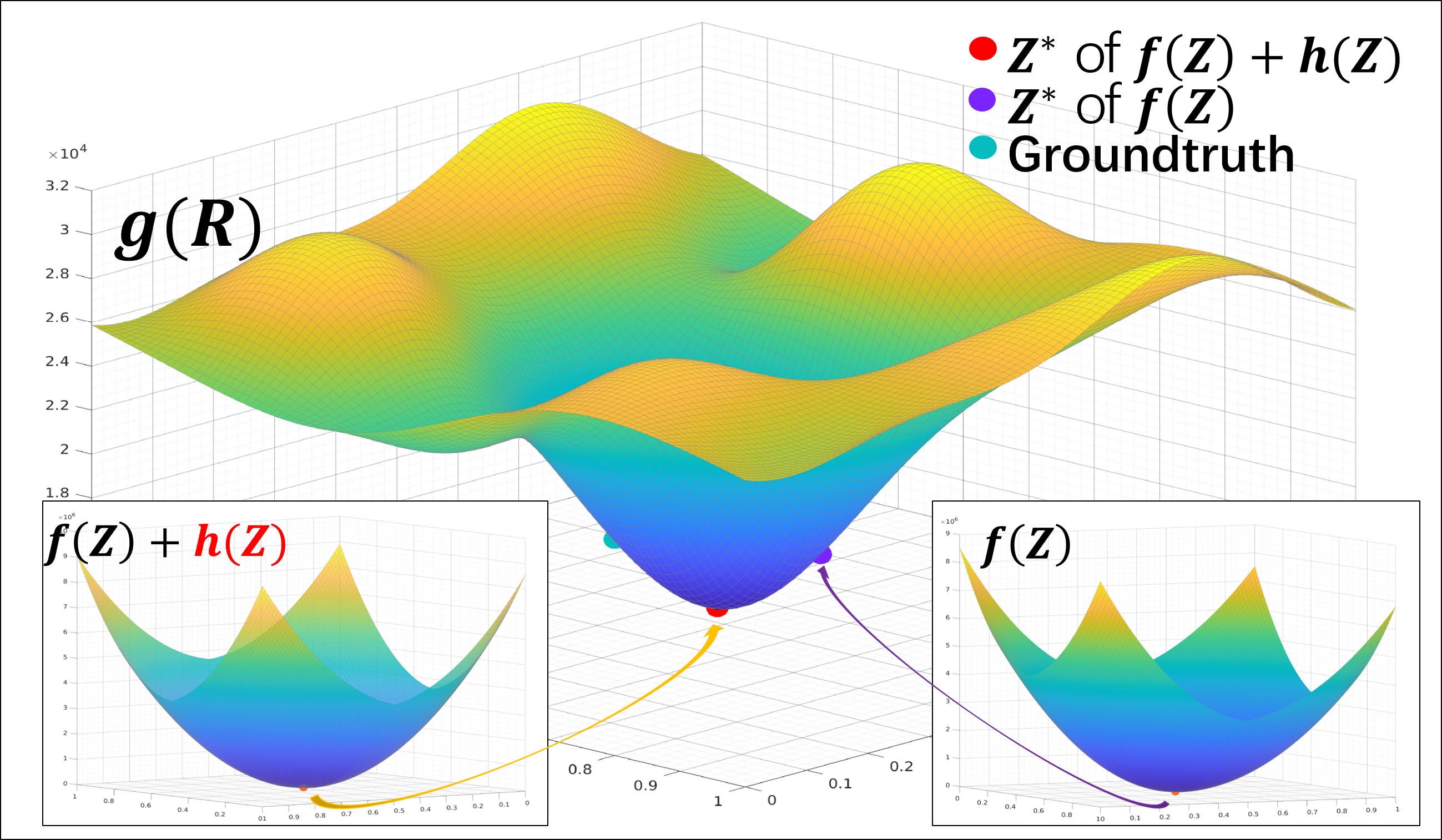}
		\caption{\label{fig:funciton} Demonstration of the effect of relaxation and rank cost. The function figure of $g(\mathbf{R})$ is plotted by applying perturbation to two dimensions of the global minimizer while fixing other dimensions. The purple and red dots represent different $\mathbf{R^*}$ recovered from $\mathbf{Z^*}$ of the Problem \ref{pro:rp} and the Problem \ref{pro:rank_cost}. }
		\vspace{-0.7cm}
	\end{figure}
	\subsubsection{Optimality}
	Firstly we verify the optimality of our method with local optimization-based algorithms. Note that $g(\mathbf{R})$ is optimized among the Stiefel manifold in the original problem. Thus we compare our method with a Riemannian manifold optimizer ($\mathbf{Riem. Opt.}$). In addition, we also compare with the Levenberg-Marquardt ($\mathbf{LM. Opt.}$) algorithm, which is the most typical solver for nonlinear least-squares problems. We carry on this experiment in MATLAB while using manopt \cite{manopt} for the manifold optimization solver and $lsqnonlin$ for the least-square problem solver. We changed the number of robots and conducted 100 tests with random measurements for each robot's amount. Note that this experiment is for verifying the optimality. Thus we add no noise to measurements. The result is presented in Fig.\ref{fig:optimality}. As it shows, our method ($\mathbf{Convex Iteration}$ and $\mathbf{SDP}$) can guarantee 100$\%$ optimality in any case, while both $\mathbf{Riem. Opt.}$ and $\mathbf{LM. Opt.}$ are possible to drop into local minimums when using random initial guesses.
	
	\begin{figure}[t]
		\centering
		\includegraphics[width=0.5\textwidth]{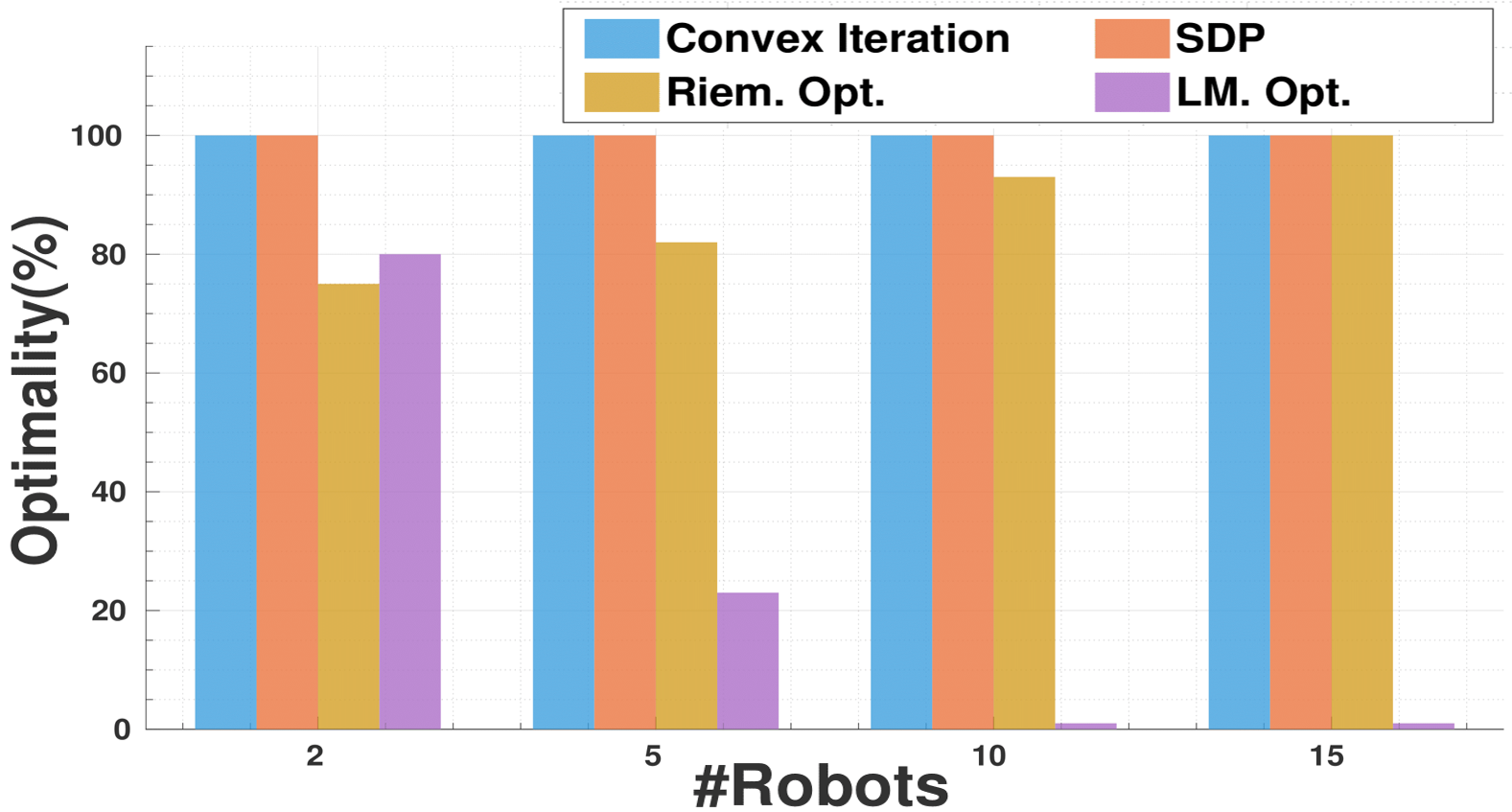}
		\caption{\label{fig:optimality} Optimality comparison results between our proposed method and local optimization-based algorithms for different robot amounts. $\mathbf{SDP}$ denotes minimizing the original $f(\mathbf{Z})$ without the rank cost. In noise-less cases, $\mathbf{SDP}$ can also obtain the optimal solution since $\textup{rank}(\mathbf{Z^*})$ is always 3.}
		\vspace{-2.0cm}
	\end{figure}
	
	\subsubsection{Runtime and Scalability}
	
	Since both our proposed problem formulations in Sec. \ref{subsec:formulation1} and Sec.\ref{subsec:formulation2} fix the number of variables by eliminating the distance variables, the computing time is only related to the number of involved robots. Thus in this experiment, we chang the robot amount and evaluate the runtime with MATLAB and C++ on different platforms. We think it is enough for a multi-robot system to adjust robots' coordination frames at 1Hz. Thus, as presented in Fig.\ref{fig:runtime}, our method has acceptable runtime on both PC and onboard computer, showing strong practicality in real multi-robot applications. In addition, as the number of robots increases, our method offers exemplary performance in scalability.
	\subsubsection{Robustness and Effectiveness of Rank Cost}
	\begin{figure}[b]
		\vspace{-0.8cm}
		\centering
		\includegraphics[width=0.5\textwidth]{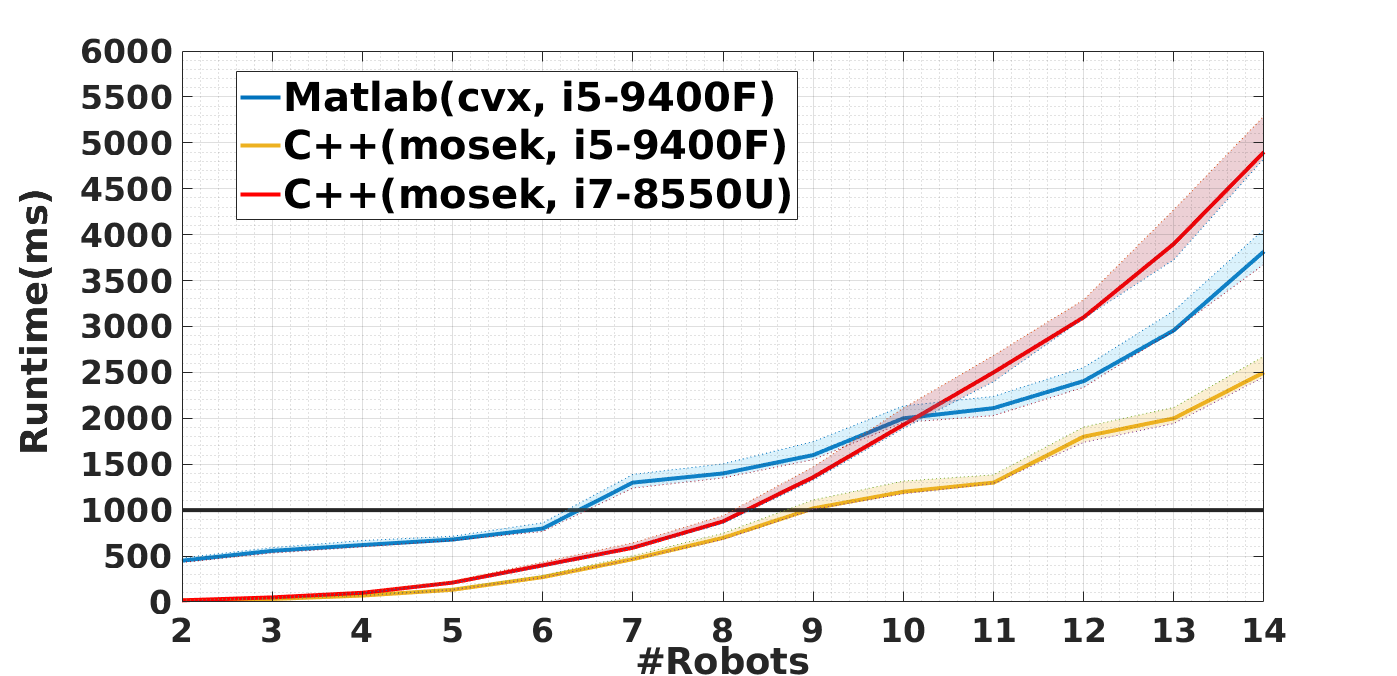}
		\vspace{-0.5cm}
		\caption{\label{fig:runtime} Comparison of runtime with different robot amounts. (solid line: mean; shaded area: 1-sigma standard deviation; black line: 1 Hz standard)}
	\end{figure}
	
	\begin{figure}[t]
		\centering
		\includegraphics[width=0.5\textwidth]{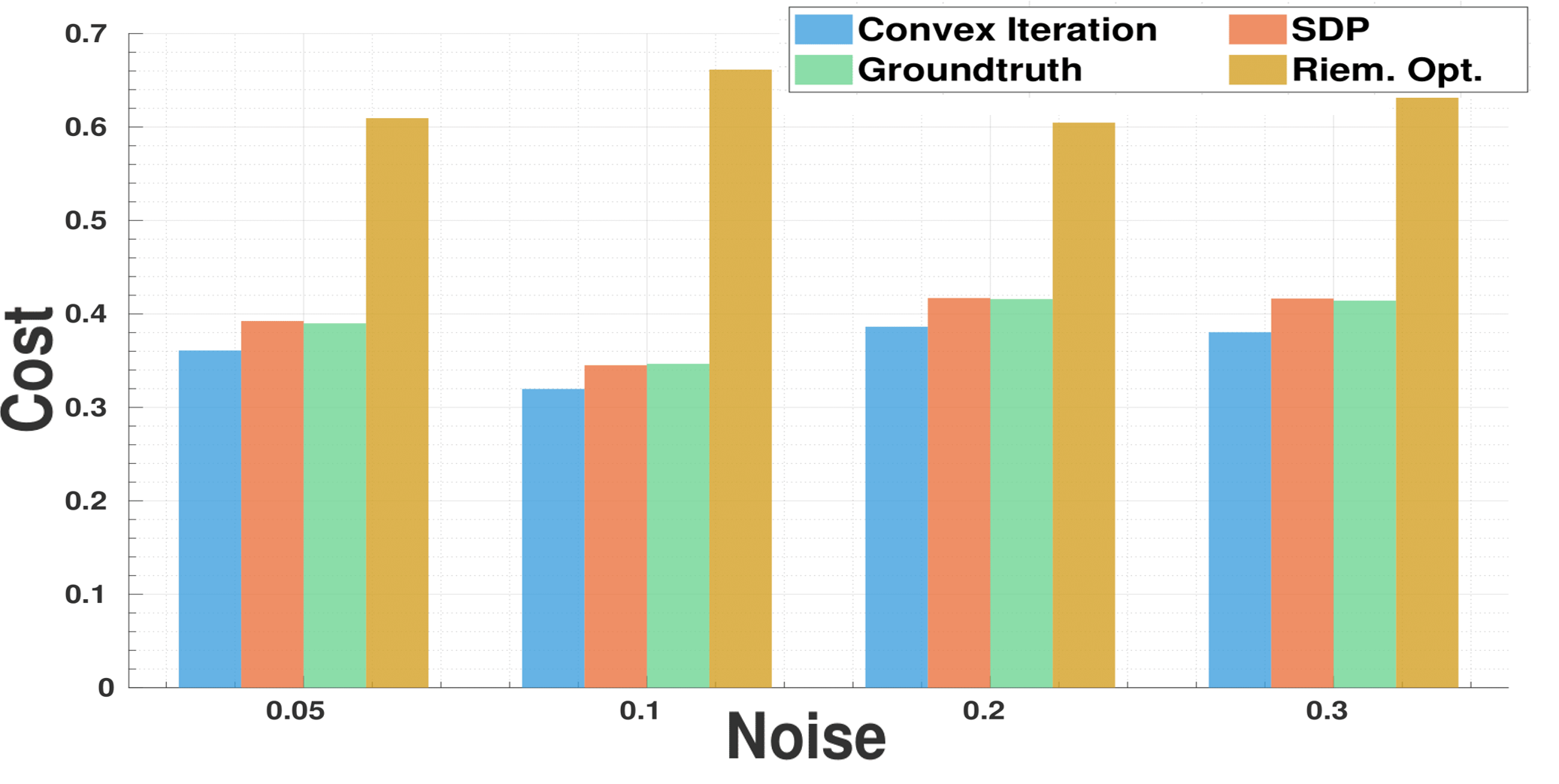}
		\vspace{-0.6cm}
		\caption{\label{fig:cost} Comparison of costs between different methods and the ground truth with varying levels of noise.}
		\vspace{-0.3cm}
	\end{figure}
	
	To present our proposed method's robustness and the effectiveness of rank cost, we add noise into simulated bearing observations and change the noise level. Fig.\ref{fig:cost} presents costs of solutions obtained by our proposed method ($\mathbf{Convex Iteration}$), pure SDP ($\mathbf{SDP}$), local optimization-based method ($\mathbf{Riem. Opt.}$) and the cost corresponding to the ground truth ($\mathbf{Groundtruth}$). $\mathbf{Riem. Opt.}$ use random matrices in the Stiefel manifold as initial values. Each bar represents the average cost of 100 random trials with five robots on simulated data using different noise levels $\sigma$. As it shows, our proposed method can consistently obtain the solution that $\textup{rank}(\mathbf{Z^*}) = 3$ and the lowest cost under noised cases, even lower than the ground  truth. In contrast, $\mathbf{SDP}$ always gets higher cost and $\mathbf{Riem. Opt.}$ always obtains the highest cost due to its erroneous local solution.

	\begin{figure}[t]
		\centering
		\includegraphics[width=0.45\textwidth]{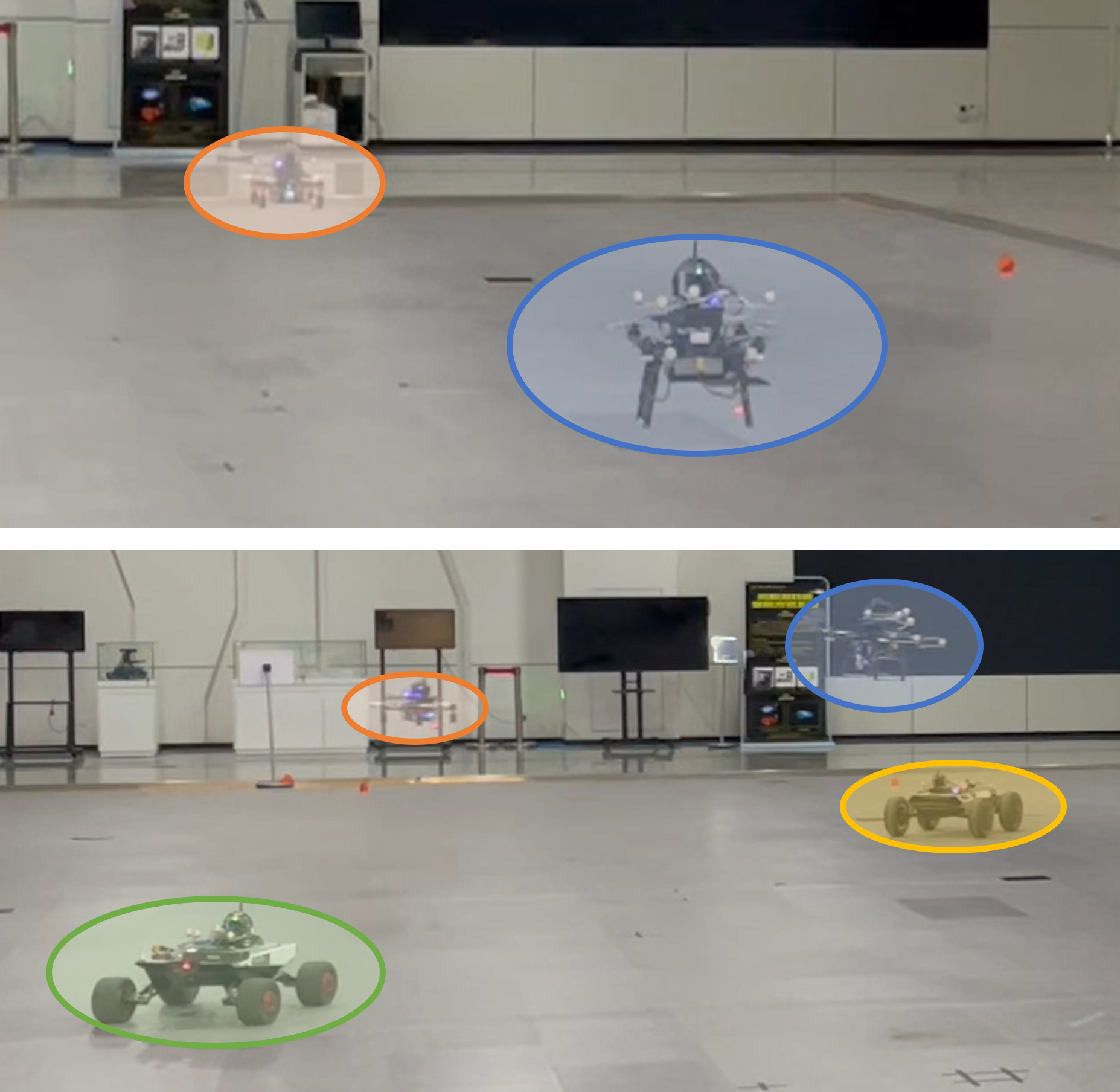}
		\vspace{-0.2cm}
		\caption{\label{fig:real} The multi-robot platforms in two real-world experiments.}
		\vspace{-0.4cm}
	\end{figure}

	\subsection{Real-world Experiments}
	Finally, we apply our proposed algorithm in the real world, as shown in Fig.\ref{fig:real}. We use platforms in our previous work as detection equipment to obtain bearing measurements, including a fish-eye camera and tagged LED marker.Firstly, we carry out experiments using two UAVs. Two robots move in 3D predetermined trajectories and observe each other. And we use motion capture as robots' local odometry and consider them the ground truth. All data, including poses and bearings, are shared with networking and collected to formulate the problem in each robot's onboard computer. 
	
	Then, to present that our method can easily fuse multiple robots' data to estimate each robot's pose consistently, we carry on four-robots experiments, including two UGVs and two UAVs. In this experiment, we drop the bearing measurements between two UAVs to verify that we can recover relative poses between two robots even if there is no direct observation between them. We transform all robots' odometry to one robot's frame, as shown in Fig.\ref{fig:exp_all}. It shows that our method can accurately recover relative poses in each experiment while the local optimization-based algorithm may drop into the local minimum. In addition, the eigenvalue, cost and residual of obtained solutions in four-robot experiments are presented in Fig.\ref{fig:cost-time}. It shows that after collecting enough data at the 5th second, our method can consistently obtain a 3-rank solution and keep the lowest cost and residual, while other methods have higher cost and residual. More details of the experiments are available in the attached video.
	
	\begin{figure}[t]
		\centering
		\includegraphics[width=0.45\textwidth]{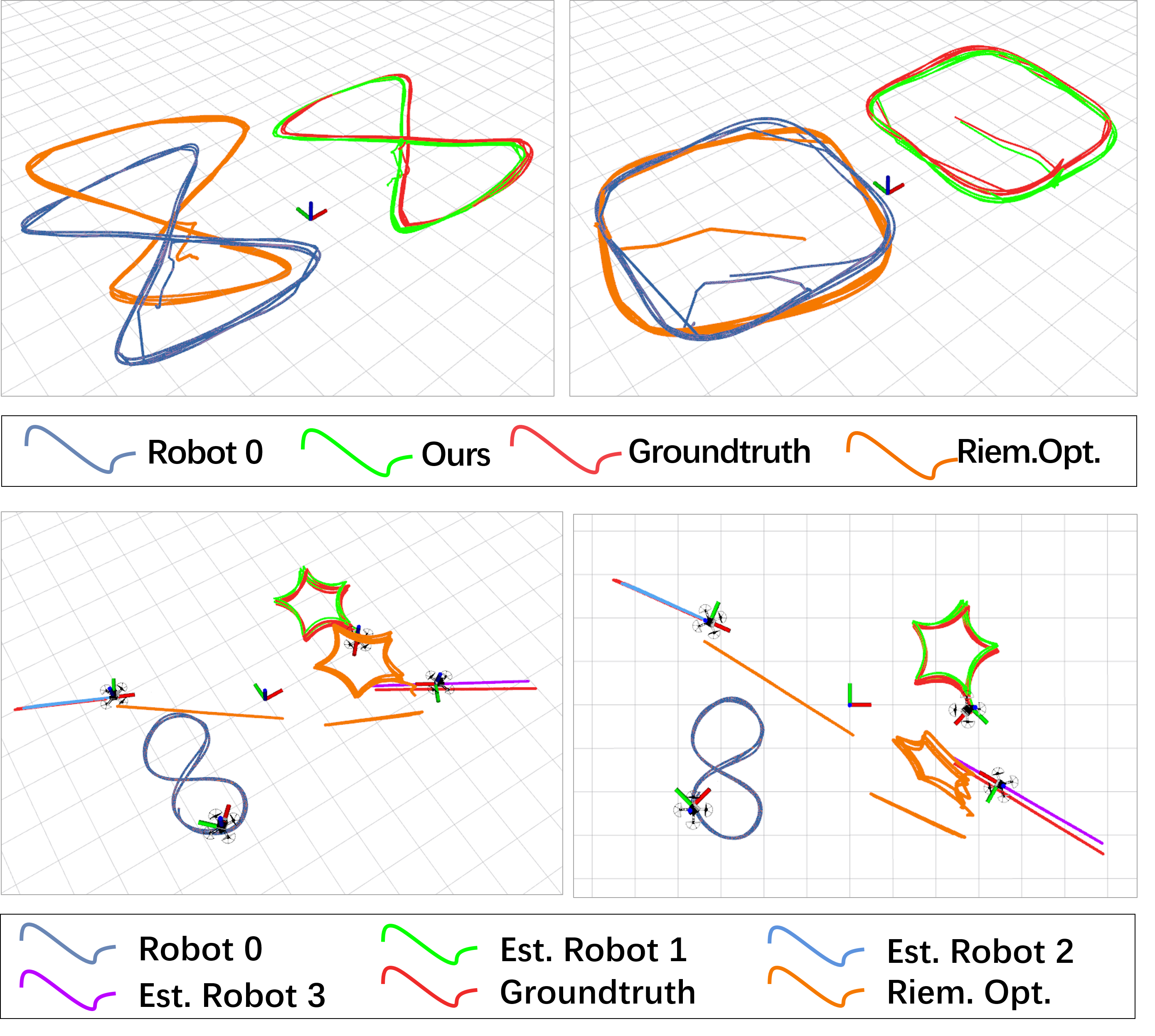}
		\vspace{-0.2cm}
		\caption{\label{fig:exp_all} Result of real-world experiments. The trajectories of robots are all presented in robot 0's coordination frame. Our method successfully recovers the most accurate relative pose, while the local optimization-based method falls into the local minimum and leads to an erroneous solution.}
		\vspace{-0.3cm}
	\end{figure}

	\begin{figure}[t]
		\centering
		\includegraphics[width=0.48\textwidth]{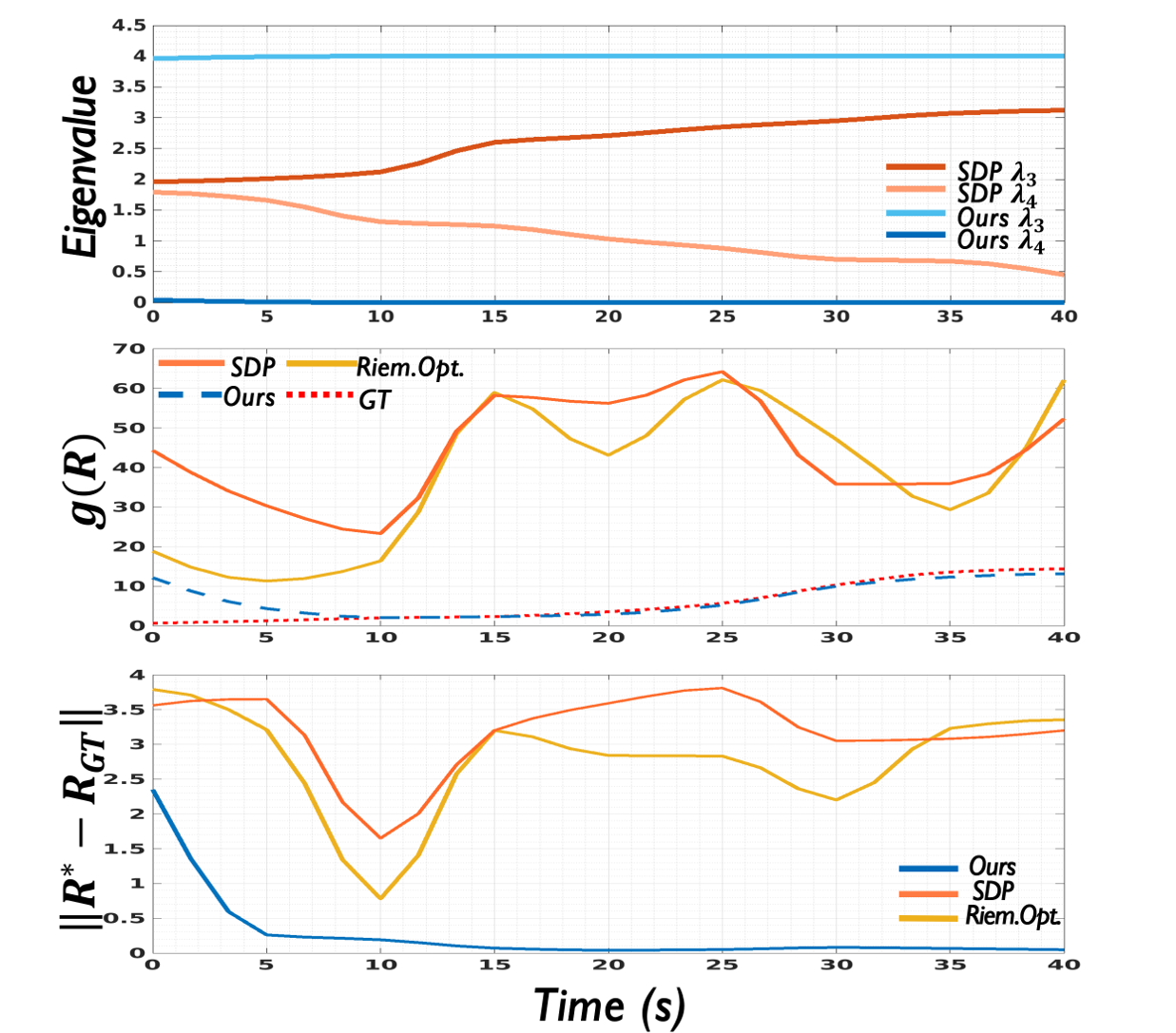}
		\vspace{-0.2cm}
		\caption{\label{fig:cost-time} Eigenvalues, costs and residuals of solutions in the four-robots test.}
		\vspace{-1.3cm}
	\end{figure}

	\section{Conclusion and future work}
	This paper proposes an accurate solver of mutual localization for multi-robot systems. Based on the partial observation graph, we use semidefinite relaxation and convex iterative optimization to provide a globally optimal consensus of reference frames for all robots. Extensive experiments on synthetic and real-world datasets show the outperforming accuracy of our method compared with local optimization-based methods. In the future, we will turn our attention to planning a suitable formation for swarm robots to meet the observability requirement of mutual localization.
	
	\bibliography{RAL2023_wyj}
\end{document}